\documentclass[11pt]{article}

\usepackage[preprint]{jmlr2e}

\ShortHeadings{Data-Dependent Worst-Case Generalization Bounds}{Leroux, Marcus and Roger}
\firstpageno{1}

\usepackage{amsmath}
\usepackage{mathtools}
\usepackage{booktabs}
\usepackage{array}
\usepackage{tabularx}
\usepackage{tikz}
\usetikzlibrary{positioning, arrows.meta, shapes.geometric, fit, backgrounds, calc}
\usepackage{enumitem}
\usepackage{xcolor}

\newtheorem{definition}{Definition}
\newtheorem{remark}{Remark}

\newtheorem{assumption}{Assumption}

\newcommand{\R}{\mathbb{R}}
\newcommand{\E}{\mathbb{E}}
\newcommand{\Prob}{\mathbb{P}}
\newcommand{\KL}{\mathrm{KL}}
\newcommand{\Rad}{\mathrm{Rad}}
\newcommand{\dimM}{\overline{\dim}_M}
\newcommand{\dimMlow}{\underline{\dim}_M}
\newcommand{\dimB}{\overline{\dim}_B}
\newcommand{\dimPH}{\dim_{PH}}
\newcommand{\PMag}{\mathrm{PMag}}
\newcommand{\Mag}{\mathrm{Mag}}
\newcommand{\W}{\mathcal{W}}
\newcommand{\WS}{\W_{S}}
\newcommand{\WSU}{\W_{S,U}}
\newcommand{\Zspace}{\mathcal{Z}}
\newcommand{\Xspace}{\mathcal{X}}
\newcommand{\Yspace}{\mathcal{Y}}
\newcommand{\Eset}{\mathcal{E}}
\newcommand{\Esigma}{\mathfrak{E}}
\newcommand{\Empirical}{\widehat{\mathcal{R}}}
\newcommand{\Risk}{\mathcal{R}}
\newcommand{\GS}{G_{S}}

\begin{document}

\title{A Survey on Data-Dependent Worst-Case Generalization Bounds}

\author{%
\name Hubert Leroux \email hubert.leroux@polytechnique.org \\
\addr \'Ecole polytechnique\\Palaiseau, France
\AND
\name Jean Marcus \email jean.marcus@polytechnique.org \\
\addr \'Ecole polytechnique\\Palaiseau, France
\AND
\name Julien Roger \email julien.roger@polytechnique.org \\
\addr \'Ecole polytechnique\\Palaiseau, France
}

\editor{}

\maketitle

\begin{abstract}
Deep neural networks generalize well despite being heavily over\-pa\-ra\-me\-te\-rized,
in apparent contradiction with classical learning theory based on uniform
convergence over fixed hypothesis spaces. Uniform bounds over the entire
parameter space are vacuous in this regime, and recent work has shown that
non-vacuous guarantees can be recovered by restricting attention to the part
of parameter space that the algorithm actually visits. This survey paper organizes
this line of work around three steps: extending PAC-Bayesian theory to random,
data-dependent hypothesis sets \citep{dupuis2024uniform}; refining the
complexity term with geometric and topological descriptors of the optimization
trajectory, including fractal dimensions, $\alpha$-weighted lifetime sums,
and positive magnitude
\citep{simsekli2020hausdorff, dupuis2023fractal, andreeva2024topological};
and replacing the resulting information-theoretic terms by stability
assumptions \citep{tuci2026stability}. We unify these contributions around
a single template inequality and a head-to-head comparison of the resulting
bounds.
\end{abstract}


\section{Introduction}\label{sec:intro}

Modern deep learning models routinely contain billions of parameters, far
more than the number of training samples, and yet they generalize well to
unseen data. This contradicts the classical bias-variance trade-off, under
which generalization error is expected to grow with model capacity
\citep{anthony2009neural}.

We work in the standard supervised setting. Let $z = (x,y) \in \Zspace =
\Xspace \times \Yspace$ be a data point drawn from an unknown distribution
$\mu_z$, and let $w \in \W \subset \R^d$ parameterize a predictor through a
loss $\ell : \R^d \times \Zspace \to \R$. The learner aims to minimize the
\emph{population risk} $\Risk(w) := \E_{z \sim \mu_z}[\ell(w,z)]$, but,
$\mu_z$ being unknown, only has access to the \emph{empirical risk}
\[
  \Empirical_S(w) := \frac{1}{n}\sum_{i=1}^n \ell(w,z_i),
  \qquad S = (z_1,\dots,z_n) \sim \mu_z^{\otimes n}.
\]
Generalization is then quantified by the \emph{uniform generalization gap}
over a hypothesis set $\W$,
\begin{equation}\label{eq:gen-gap}
  \GS(\W) := \sup_{w \in \W} \bigl(\Risk(w) - \Empirical_S(w)\bigr),
\end{equation}
and bounding this quantity when $\W$ is the full parameter space gives
vacuous guarantees in deep learning \citep{zhang2016understanding,
nagarajan2019uniform}. The crucial observation behind recent progress is
that a learning algorithm does not visit the entire parameter space: it
traces a specific, data-dependent trajectory $\WS \subset \W$ (see Figure~\ref{fig:vacuity}) that depends
on the dataset, the optimizer, and the hyper-parameters. Replacing $\W$ by
$\WS$ in \eqref{eq:gen-gap} captures the part of $\R^d$ the algorithm
actually visits, rather than the worst-case complexity of every model that
could have been trained. Several recent works follow this path
\citep{simsekli2020hausdorff, dupuis2023fractal, dupuis2024uniform,
andreeva2024topological, tuci2026stability}, but use different formalisms:
fractal geometry, PAC-Bayesian theory, and algorithmic stability.

\begin{figure}[t]
  \centering
  \includegraphics[width=0.55\textwidth]{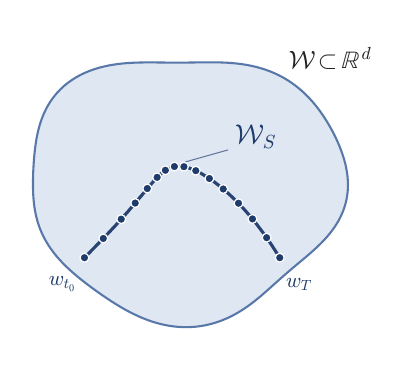}
  \caption{The data-dependent hypothesis set $\WS$. The ambient parameter
  space $\W \subset \R^d$ is high-dimensional, but a learning algorithm
  only visits a thin, data-dependent subset $\WS$ along its optimization
  trajectory. The bounds in \S\ref{sec:pac-bayes}--\S\ref{sec:stability}
  control $\GS(\WS)$ rather than $\GS(\W)$, exploiting the fact that
  $\WS$ has much lower effective complexity than $\W$ itself.}
  \label{fig:vacuity}
\end{figure}

\subsection{Anatomy of a data-dependent bound}\label{sec:anatomy}

A useful observation, made informally in \citet{andreeva2024topological},
is that almost all data-dependent bounds discussed below share the same
\emph{template}: with probability at least $1-\zeta$,
\begin{equation}\label{eq:template}
  \GS(\WS) \;\lesssim\;
  \sqrt{\,\frac{\;
    \underbrace{\mathrm{Complexity}(\WS)}_{\text{geometry of trajectory}}
    \;+\;
    \underbrace{\mathrm{IT}(S, \WS)}_{\text{data dependence}}
    \;+\;
    \log(1/\zeta)\,}{n}\,}.
\end{equation}
Each of the three steps mentioned above acts on a different term (see Figure~\ref{fig:roadmap}). PAC-Bayes
on random sets (\S\ref{sec:pac-bayes}) provides the foundation that makes
\eqref{eq:template} possible. Topological complexity (\S\ref{sec:topology})
refines the \emph{Complexity} term, replacing covering numbers with
geometrically informative quantities. Algorithmic stability
(\S\ref{sec:stability}) removes the \emph{IT} term, which is intractable for
practical optimizers. Figure~\ref{fig:roadmap} summarizes the structure of
the survey, and Section~\ref{sec:synthesis} closes with a head-to-head
comparison of the bounds reached at each stage (Table~\ref{tab:bounds}).

\begin{figure}[t]
\centering
\begin{tikzpicture}[
  node distance=10mm and 6mm,
  every node/.style = {font=\small},
  block/.style = {
    rectangle, rounded corners=2pt, draw, thick,
    minimum height=10mm, minimum width=30mm, align=center,
    fill=blue!5
  },
  problem/.style = {
    rectangle, rounded corners=2pt, draw, thick,
    minimum height=10mm, minimum width=30mm, align=center,
    fill=red!10
  },
  movement/.style = {
    rectangle, draw=none, align=center,
    font=\small\itshape, text=blue!60!black
  },
  arrow/.style = {-{Stealth[length=2.5mm]}, thick}
]

\node[problem] (vacuous) {Vacuous bound\\over fixed $\W$};
\node[movement, right=of vacuous] (m1) {PAC-Bayes\\on random sets};
\node[block, right=of m1] (datadep) {Data-dependent\\bound on $\WS$};
\node[movement, below=of datadep] (m2) {Topological\\complexity};
\node[block, below=of m2] (refined) {Refined\\complexity term};
\node[movement, left=of refined] (m3) {Algorithmic\\stability};
\node[block, left=of m3, fill=green!10] (final) {IT-free,\\computable bound};

\draw[arrow] (vacuous) -- (m1);
\draw[arrow] (m1) -- (datadep);
\draw[arrow] (datadep) -- (m2);
\draw[arrow] (m2) -- (refined);
\draw[arrow] (refined) -- (m3);
\draw[arrow] (m3) -- (final);

\node[font=\scriptsize, below=1pt of m1] {\S\ref{sec:pac-bayes}};
\node[font=\scriptsize, right=1pt of m2] {\S\ref{sec:topology}};
\node[font=\scriptsize, below=1pt of m3] {\S\ref{sec:stability}};

\end{tikzpicture}
\caption{Roadmap of the survey. Each step acts on a different term of the
template inequality~\eqref{eq:template}. \S\ref{sec:fail} establishes the
starting point (left, red); \S\ref{sec:synthesis} closes with a head-to-head
comparison of the bounds reached at each stage.}
\label{fig:roadmap}
\end{figure}
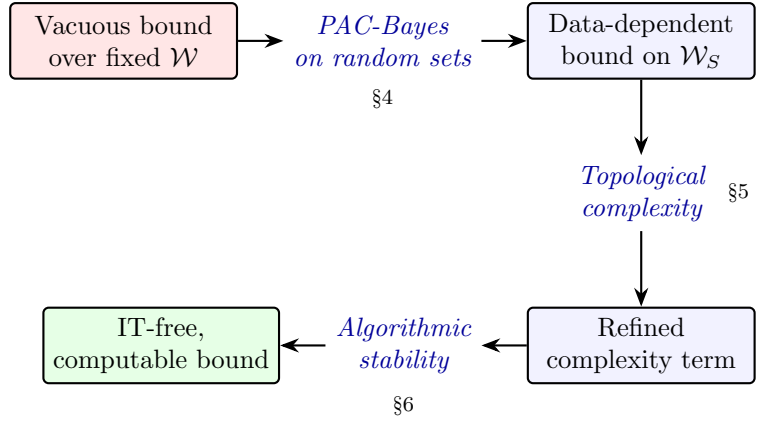

\section{Preliminaries}\label{sec:prelim}

We fix notation, list the assumptions used in later sections, and define the
complexity measures that appear in our bounds. Theorems below refer back to
these by number, so that proof statements remain compact.

\subsection{Notation}

We write $\W \subset \R^d$ for a generic parameter space. A \emph{fixed}
hypothesis set is denoted $\W$; a \emph{random}, data-dependent set is
denoted $\WS$ (depending on the dataset $S$) or $\WSU$ (depending on $S$
and on the algorithm's randomness $U$). The data $S = (z_1,\dots,z_n)$ is
i.i.d.\ from $\mu_z$ on $\Zspace$. We use $\Risk$, $\Empirical_S$, $\GS$
as defined in \S\ref{sec:intro}. For a measurable function $\phi$ on a product
space, $\E_S$ denotes expectation over $S \sim \mu_z^{\otimes n}$.

\subsection{Standing assumptions}\label{sec:assumptions}

The seven assumptions below appear in various combinations throughout the
paper, and we gather them here once so that later theorems can easily refer to them
by number. The first four are regularity conditions
on the loss and on the hypothesis sets, used in
\S\ref{sec:fail}--\S\ref{sec:pac-bayes}.

\begin{assumption}[Bounded loss]\label{ass:bounded}
$\ell : \R^d \times \Zspace \to \R$ is measurable and $\ell(w,z) \in [0,B]$
for some $B > 0$.
\end{assumption}

\begin{assumption}[Supremum measurability]\label{ass:sup-meas}
For any coefficients $b, a_1, \dots, a_n \in \R$, the map
\[
  (\W, S) \longmapsto \sup_{w \in \W}\sum_{i=1}^n
    \bigl(a_i\, \ell(w,z_i) - b\,\Risk(w)\bigr)
\]
is jointly measurable with respect to the product $\sigma$-algebra
$\Esigma \otimes \mathcal{F}^{\otimes n}$ (notation as in
\S\ref{sec:random-sets}).
\end{assumption}

\begin{assumption}[Measurable covering numbers]\label{ass:meas-cover}
The covering numbers (Definition~\ref{def:covering}) used in this paper are
measurable with respect to $\mathcal{F}^{\otimes n} \otimes \Esigma$.
\end{assumption}

To rigorously transition from finite-scale covering numbers to asymptotic fractal dimensions in later bounds, we require the finite-resolution approximation of the dimension to converge uniformly with respect to the dataset size $n$.

\begin{assumption}[Uniformity in probability]\label{ass:uniformity}
For all $\epsilon > 0$,
\[
  \sup_{n \in \mathbb{N}^\star}
  \int_{\Zspace^n}
    Q_S\!\left(
      \sup_{0<r<\delta}
        \frac{\log |N^{\vartheta_S}_r(\WS)|}{\log(1/r)}
        \;-\; \dimB^{\vartheta_S}(\WS) \geq \epsilon
    \right) d\mu_z^{\otimes n}(S)
  \;\xrightarrow[\delta\to 0]{}\; 0.
\]
\end{assumption}

The remaining three assumptions belong to a different family: they
constrain how the \emph{algorithm} reacts to a small perturbation of the
training set. They are used only in \S\ref{sec:stability}, where they
replace the information-theoretic terms that appear in earlier bounds.

The notion of uniform stability was originally introduced by
\citet{bousquet2002stability}; the argument-stability variant we use here
follows \citet{bassily2020stability}.

\begin{assumption}[Uniform argument stability,
\citealp{bassily2020stability}]\label{ass:uas}
A randomized algorithm $A : (S,U) \mapsto \R^d$ is $\beta$-uniformly
argument-stable if for all datasets $S, S'$ differing in at most one element,
$\E_U[\|A(S,U) - A(S',U)\|] \leq \beta$.
\end{assumption}

While Assumption~\ref{ass:uas} characterizes the stability of a single algorithmic output, bounding the worst-case generalization over an entire optimization trajectory requires lifting this standard notion to a property of the data-dependent hypothesis set itself.

\begin{assumption}[Trajectory stability,
\citealp{tuci2026stability}]\label{ass:traj}
The algorithm is $\beta_n$-trajectory-stable if for any $J \in \mathbb{N}^\star$
and any data-dependent selection $\omega : \WSU \times \Zspace^n \to w \in
\WSU$, there exists a measurable map $\omega'$ such that for all $S, S'$
differing in $J$ points and all $z$,
\[
  \E_U\bigl[|\ell(\omega(\WSU,S),z) - \ell(\omega'(\mathcal{W}_{S',U},\omega(\WSU,S)),z)|\bigr]
  \;\leq\; \beta_n J.
\]
\end{assumption}

\begin{assumption}[Lipschitz on random sets]\label{ass:lip-rs}
There exists a measurable map $(S,U) \mapsto L_{S,U}$ such that
$|\ell(w,z) - \ell(w',z)| \leq L_{S,U}\|w - w'\|$ for all $w,w' \in \WSU$
and $z$.
\end{assumption}

\begin{table}[t]
\centering
\small
\begin{tabular}{@{}llll@{}}
\toprule
\# & Name & Statement & Used in \\
\midrule
A1 & Bounded loss & $\ell \in [0,B]$ & \S\ref{sec:fail}, \S\ref{sec:pac-bayes}, \S\ref{sec:topology}, \S\ref{sec:stability} \\
A2 & Sup-measurability & joint measurability of $\sup$ & \S\ref{sec:pac-bayes} \\
A3 & Measurable covering & $N_\delta(\W)$ measurable & \S\ref{sec:pac-bayes} \\
A4 & Uniformity in probability & uniform conv.\ in probability & \S\ref{sec:pac-bayes} (fractal) \\
A5 & Uniform argument stability & $\beta$-stable algorithm & \S\ref{sec:stability} \\
A6 & Trajectory stability & $\beta_n$-stable random set & \S\ref{sec:stability} \\
A7 & Lipschitz on $\WSU$ & $L_{S,U}$ random Lipschitz constant & \S\ref{sec:stability} \\
\bottomrule
\end{tabular}
\caption{Summary of standing assumptions used throughout the survey.}
\label{tab:assumptions}
\end{table}

Table~\ref{tab:assumptions} summarizes the seven assumptions together with
the sections in which each is invoked.

\subsection{Complexity measures at a glance}\label{sec:complexities}

The bounds we are about to discuss control the generalization gap through
quantities measuring, in one form or another, the ``size'' or ``richness''
of the hypothesis set. These come in three flavours and are introduced
below in increasing geometric sophistication: covering-based quantities
that count points at a given resolution, fractal dimensions that capture
how that count scales with resolution, and topological or magnitude-based
descriptors that exploit the internal structure of the trajectory.
Table~\ref{tab:complexity} gives a high-level summary.

We begin with the covering-based family, which underlies most classical
bounds. A covering of $\W$ at scale $\delta$ is a finite set of $\delta$-balls
whose union contains $\W$; the smaller $\delta$, the finer the cover.

\begin{definition}[Pseudometric space and covering number]\label{def:covering}
$(\Xspace, \vartheta)$ is a \emph{pseudometric space} if $\vartheta$ is symmetric,
satisfies the triangle inequality, and $\vartheta(x,x)=0$. For a compact
$(\Xspace,\vartheta)$ and $\delta > 0$, $N^{\vartheta}_\delta(\Xspace)$ is the minimal
set of centers covering $\Xspace$ by closed $\delta$-balls; $|N^{\vartheta}_\delta(\Xspace)|$
is the corresponding \emph{covering number}. We omit $\vartheta$ when it is
the Euclidean distance on $\R^d$.
\end{definition}

The fractal family captures how $|N^\vartheta_\delta(\W)|$ scales as
$\delta \to 0$. Sets that look low-dimensional at small scales have small
Minkowski dimension, which makes them ``effectively'' smaller than the
ambient dimension $d$ even when they live in $\R^d$.

\begin{definition}[Minkowski (box-counting) dimension]\label{def:minkowski}
For $\W \subset \R^d$,
\[
  \dimMlow \W := \liminf_{\delta\to 0}\frac{\log N_\delta(\W)}{-\log \delta},
  \qquad
  \dimM \W := \limsup_{\delta\to 0}\frac{\log N_\delta(\W)}{-\log \delta}.
\]
When the two coincide, the common value is $\dim_M \W$. The hierarchy
$0 \leq \dim_H \W \leq \dimMlow \W \leq \dimM \W \leq d$ holds, where $\dim_H$
is the Hausdorff dimension.
\end{definition}

\begin{definition}[Persistent-homology dimension,
\citealp{schweinhart2020fractal}]\label{def:ph-dim}
Let $(X, \rho)$ be a finite pseudometric space, $T$ a minimum spanning tree,
and $\alpha \geq 0$. The \emph{$\alpha$-weighted lifetime sum} is
\[
  E^\rho_\alpha(X) := \sum_{e \in T} |e|^\alpha.
\]
For a compact $(A,\rho)$,
$\dimPH^\rho(A) := \inf\{\alpha \geq 0 : \exists C>0,\ \forall Y\subset A
\text{ finite},\ E^\rho_\alpha(Y) \leq C\}$.
\end{definition}

The third family is topological and geometric. Unlike covering numbers,
which are insensitive to how points within a cover are arranged, these
quantities respond to the internal structure of $\W$: clustered iterates
contribute short MST edges and small magnitude, whereas iterates that
spread out across $\R^d$ contribute long edges and large magnitude
(Figures~\ref{fig:trajectory} and ~\ref{fig:magnitude} illustrate this on a simulated trajectory).

\begin{definition}[$\alpha$-weighted lifetime sums]\label{def:lifetime}
We use $E^\rho_\alpha(X)$ as a stand-alone complexity measure, not only as a
proxy for $\dimPH^\rho$.
\end{definition}

\begin{definition}[Magnitude, \citealp{leinster2013magnitude}]\label{def:magnitude}
Let $(X,\rho)$ be finite and $s>0$. A \emph{weighting} is $\beta : X \to \R$
satisfying $\sum_{b \in X} e^{-s\rho(a,b)}\beta(b) = 1$ for all $a \in X$. The
magnitude is $\Mag^\rho(sX) := \sum_a \beta(a)$.
\end{definition}

\begin{definition}[Positive magnitude,
\citealp{andreeva2024topological}]\label{def:posmag}
With $\beta_s$ the weighting above,
$\PMag^\rho(sX) := \sum_a \beta_s(a)^+$, where $x^+ := \max(x,0)$.
\end{definition}

Finally we record three quantities that are not complexities of $\W$
itself but appear repeatedly in our bounds: the Rademacher complexity (a
statistical analogue of covering numbers), the total mutual information
$I_\infty$ (which measures how strongly $\WS$ depends on the data $S$),
and the Kullback--Leibler divergence (the data-dependence term in all
PAC-Bayesian bounds).

\begin{definition}[Empirical Rademacher complexity]\label{def:rad} 
The empirical Rademacher complexity is defined as: 
$\Rad_S(\W) = \E_\sigma\bigl[\sup_{w\in\W} \tfrac{1}{n}\sum_{i=1}^n
\sigma_i \ell(w,z_i)\bigr]$, where $\sigma_i$ are i.i.d.\ uniform on
$\{-1,+1\}$.
\end{definition}

\begin{definition}[Total mutual information]\label{def:itinfo}
For random elements $X,Y$ on a probability space $(\Omega,\mathcal{F},\Prob)$,
\[
  I_\infty(X,Y) := \log \sup_A \frac{\Prob_{X,Y}(A)}{\Prob_X \otimes \Prob_Y(A)}.
\]
\end{definition}

\begin{definition}[Kullback-Leibler divergence]\label{def:kl}
For probability measures $\mu \ll \nu$ on $\W$,
$\KL(\mu \| \nu) := \int \log(d\mu/d\nu)\, d\mu$, with $\KL(\mu\|\nu) = +\infty$
otherwise.
\end{definition}

\begin{table}[t]
\label{tab:complexity}
\centering
\small
\begin{tabular}{@{}llll@{}}
\toprule
Quantity & Family & What it captures & Reference \\
\midrule
$|N^\vartheta_\delta(\W)|$ & metric & cover at scale $\delta$ & -- \\
$\Rad_S(\W)$ & statistical & ability to fit noise & \citet{mohri2018foundations} \\
$\dim_M(\W)$ & fractal & asymptotic covering scaling & \citet{simsekli2020hausdorff} \\
$\dimB^{\vartheta_S}(\WS)$ & fractal (data-dep.) & box-counting on random set & \citet{dupuis2024uniform} \\
$\dimPH^\rho(\W)$ & topological & persistent-homology dim. & \citet{schweinhart2020fractal} \\
$E^\rho_\alpha(\W)$ & topological & $\alpha$-weighted MST length & \citet{andreeva2024topological} \\
$\Mag^\rho(s\W)$ & geometric & effective ``volume'' & \citet{leinster2013magnitude} \\
$\PMag^\rho(s\W)$ & geometric & positive part of magnitude & \citet{andreeva2024topological} \\
\bottomrule
\end{tabular}
\caption{Complexity measures used in this paper.}
\end{table}

\section{Why uniform bounds over the whole parameter space fail}\label{sec:fail}

This section establishes the starting point of the survey. Classical bounds
based on covering numbers and on Rademacher complexity, applied to a fixed
$\W$, are textbook results, but they are vacuous in the deep-learning regime
\citep{zhang2016understanding, nagarajan2019uniform}: when $\W = \R^d$ and
$d$ is large, the right-hand sides explode. \citet{nagarajan2019uniform}
prove this is not an artifact of looseness: \emph{any} uniform bound over a
fixed, data-independent set must be vacuous in deep learning, because such
sets contain models that simply memorize random noise. We include the two
classical bounds below so that the rest of the paper can be read as a
reaction to their failure.

\subsection{Covering-based bound for fixed $\W$}

We first relate the generalization gap to the Minkowski dimension of a
fixed, finite-diameter $\W$.

\begin{theorem}[\citealp{simsekli2020hausdorff}, Theorem~1]
\label{thm:minkowski}
Under Assumption~\ref{ass:bounded}, suppose $\ell$ is $L$-Lipschitz in $w$ and
$\W \subset \R^d$ has finite diameter. For $n$ large enough, with probability
at least $1 - \zeta$ over $S \sim \mu_z^{\otimes n}$,
\begin{equation}\label{eq:thm-minkowski}
  \GS(\W) \;\leq\;
    B\sqrt{\frac{2\,\dimM(\W)\log(nL^2)}{n} + \frac{\log(1/\zeta)}{n}}.
\end{equation}
\end{theorem}

\noindent\emph{Proof sketch.} Cover $\W$ with a finite $\delta$-net of size
$|N_\delta(\W)|$. Lipschitzness controls the discretization error by $2L\delta$.
A union bound combined with Hoeffding's inequality on the net, together with
the Minkowski-dimension scaling of $|N_\delta(\W)|$, yields the bound. See
the original reference for the full proof.

\subsection{Rademacher-based bound for fixed $\W$}

\begin{theorem}[\citealp{mohri2018foundations}, Theorem~3.3]
\label{thm:rademacher}
Under Assumption~\ref{ass:bounded}, for any $\zeta \in (0,1)$,
\begin{equation}
  \Prob_{S \sim \mu_z^{\otimes n}}\!
    \left(\GS(\W) \leq 2\Rad_S(\W) + 3B\sqrt{\frac{\log(1/\zeta)}{2n}}\right)
  \;\geq\; 1-\zeta.
\end{equation}
This slightly improved version is given by \citet[Theorem~1]{dupuis2024uniform}.
\end{theorem}

\noindent\emph{Proof sketch.} Symmetrization introduces Rademacher random
variables; a careful application of McDiarmid's inequality turns the
expectation bound into a high-probability bound.

\section{Restoring data-dependence: PAC-Bayes on random sets}\label{sec:pac-bayes}

The classical PAC-Bayesian framework partly addresses the failure of
\S\ref{sec:fail} by averaging over a posterior $Q_S$ instead of taking a
worst-case parameter, but it still operates on a fixed parameter space.
We present below the recent extension of \citet{dupuis2024uniform} that
allows the support itself to be a random set.

\subsection{Classical PAC-Bayes in one statement}\label{sec:classical-pb}

PAC-Bayesian theory bounds the average error of a posterior $Q_S$ relative
to a fixed prior $P$, blending Bayesian intuition with frequentist validity.
Several historical formulations exist: \citet{catoni2003pac} established a
tight bound under bounded loss, \citet{germain2009pac} gave a general
expectation-form result, and \citet{rivasplata2020pac} showed that the same
proof yields a \emph{disintegrated} bound that holds pointwise in
$w \sim Q_S$. We package the three under a single statement.

\begin{theorem}[\citealp{catoni2003pac, germain2009pac, rivasplata2020pac}]
\label{thm:classical-pb}
Let $\zeta \in (0,1)$ and $\phi : \W \times \Zspace^n \to \R$ be measurable
with $e^\phi \in L^1(P \otimes \mu_z^{\otimes n})$. Let $P$ be a prior and
$Q_S$ a Markov kernel with $Q_S \ll P$ and $\phi(\cdot,S) \in L^1(Q_S)$.
Then:
\begin{enumerate}[label=(\roman*), leftmargin=2.5em]
\item (\textbf{Expectation,} \citealp{germain2009pac}) With probability at least
$1-\zeta$ over $S \sim \mu_z^{\otimes n}$,
\begin{equation}\label{eq:germain}
  \E_{w\sim Q_S}[\phi(w,S)] \;\leq\;
    \KL(Q_S \| P) + \log\tfrac{1}{\zeta}
    + \log \E_S \E_{w\sim P}\!\bigl[e^{\phi(w,S)}\bigr].
\end{equation}
\item (\textbf{Disintegrated,} \citealp{rivasplata2020pac}) With probability at
least $1-\zeta$ over $(S, w \sim Q_S)$,
\begin{equation}\label{eq:rivasplata}
  \phi(w,S) \;\leq\; \log\tfrac{dQ_S}{dP}(w) + \log\tfrac{1}{\zeta}
    + \log \E_S \E_{w\sim P}\!\bigl[e^{\phi(w,S)}\bigr].
\end{equation}
\end{enumerate}
The earlier bound of \citet{catoni2003pac} is recovered by specializing
\eqref{eq:germain} to $\phi(w,S) = \lambda\bigl(\Risk(w) - \Empirical_S(w)\bigr)$
with $\lambda > 0$: under Assumption~\ref{ass:bounded}, the moment-generating
term is at most $\lambda^2 B^2/(8n)$ by Hoeffding, and dividing by $\lambda$
yields
$\E_{w\sim Q_S}[\Risk(w)] \leq \E_{w\sim Q_S}[\Empirical_S(w)]
+ (\KL(Q_S\|P) + \log(1/\zeta))/\lambda + \lambda B^2/(8n)$.
\end{theorem}

\begin{remark}
The proof relies on Donsker--Varadhan's variational formula for KL divergence;
see the cited references for details. The disintegrated form
\eqref{eq:rivasplata} sharpens \eqref{eq:germain} by avoiding the average over
$Q_S$ at the cost of being a statement over the joint distribution
$(S, w\sim Q_S)$.
\end{remark}

\subsection{From single hypotheses to random sets}\label{sec:random-sets}

We now describe the extension to data-dependent hypothesis \emph{sets},
following \citet{dupuis2024uniform}. While formally a direct consequence of
classical PAC-Bayes (the parameter space is replaced by a space of subsets),
the construction is specifically designed to handle the data-dependent
geometries inherent in modern deep learning.

Let $\mathcal{P}(\R^d)$ denote the power set of $\R^d$. Fix a collection
$\Eset \subseteq \mathcal{P}(\R^d)$ of admissible hypothesis sets, equipped
with a $\sigma$-algebra $\Esigma$. We replace the random parameter
$w \in \R^d$ by a random hypothesis set $\W \in \Eset$.

\begin{definition}[Priors and posteriors on random sets]
\label{def:rand-priors}
A \emph{prior} $P$ is a data-independent probability distribution on
$(\Eset, \Esigma)$. A \emph{family of posteriors} $(Q_S)_{S \in \Zspace^n}$
is a Markov kernel on $\Eset \times \Zspace^n$, with $Q_S \ll P$ almost surely
under $\mu_z^{\otimes n}$.
\end{definition}

The random-set version of Theorem~\ref{thm:classical-pb} follows by replacing
$w$ with $\W$ throughout the proof.

\subsection{The master bound}\label{sec:master-bound}

\begin{theorem}[\citealp{dupuis2024uniform}, PAC-Bayes on random sets]
\label{thm:master}
Let $\Phi : \Eset \times \Zspace^n \to \R$ be measurable. Under the conditions
of Definition~\ref{def:rand-priors}, for any $\zeta \in (0,1)$:
\begin{enumerate}[label=(\roman*), leftmargin=2.5em]
\item (\textbf{Expectation form}) With probability at least $1-\zeta$,
\begin{equation}\label{eq:master-exp}
  \E_{\W \sim Q_S}[\Phi(\W,S)]
  \;\leq\;
  \KL(Q_S \| P) + \log\tfrac{1}{\zeta}
  + \log \E_S \E_{\W\sim P}\!\bigl[e^{\Phi(\W,S)}\bigr].
\end{equation}
\item (\textbf{Disintegrated form}) With probability at least $1-\zeta$,
\begin{equation}\label{eq:master-dis}
  \Phi(\W,S)
  \;\leq\;
  \log\tfrac{dQ_S}{dP}(\W) + \log\tfrac{1}{\zeta}
  + \log \E_S \E_{\W\sim P}\!\bigl[e^{\Phi(\W,S)}\bigr].
\end{equation}
\end{enumerate}
Both hold whenever the expectations are well defined.
\end{theorem}

The remaining theorems specialize $\Phi$ and bound the moment-generating
term to obtain generalization guarantees in terms of concrete complexity
measures.

\subsection{Specialization: data-dependent Rademacher and covering bounds}

A symmetrization argument applied to \eqref{eq:master-exp}--\eqref{eq:master-dis}
gives a Rademacher-type bound on random sets without any boundedness
assumption on the loss \citep[][Theorem 4]{dupuis2024uniform}. Under
Assumption~\ref{ass:bounded}, this further specializes to a clean
data-dependent Rademacher bound:

\begin{theorem}[\citealp{dupuis2024uniform}, data-dependent Rademacher]
\label{thm:dd-rademacher}
Under Assumptions~\ref{ass:bounded} and \ref{ass:sup-meas}, for any $\lambda>0$,
with probability at least $1-\zeta$:
\begin{equation}\label{eq:dd-rad-exp}
  \E_{\W\sim Q_S}[\GS(\W)]
  \;\leq\;
  \E_{\W\sim Q_S}\bigl[2\Rad_S(\W)\bigr]
  + \frac{\KL(Q_S\|P) + \log(1/\zeta)}{\lambda}
  + \lambda\,\frac{9B^2}{8n},
\end{equation}
and the disintegrated counterpart with $\log\frac{dQ_S}{dP}(\W)$ in place of
$\KL(Q_S\|P)$.
\end{theorem}

When $\ell$ is also Lipschitz, replacing $\Rad_S$ by a covering-number
argument gives an explicit geometric bound.

\begin{theorem}[\citealp{dupuis2024uniform}, covering bound on random sets]
\label{thm:dd-covering}
Under Assumptions~\ref{ass:bounded}, \ref{ass:sup-meas}, and \ref{ass:meas-cover},
suppose $\ell$ is $L$-Lipschitz in $w$ and $\W$ is $P$-a.s.\ bounded. There
exists $C > 0$ such that for any $\lambda, \delta > 0$, with probability at
least $1-\zeta$:
\begin{equation}\label{eq:dd-cover}
  \GS(\W)
  \;\leq\;
  2L\delta + 2B\sqrt{\frac{2\log|N_\delta(\W)|}{n}}
  + \frac{\log\frac{dQ_S}{dP}(\W) + \log(1/\zeta)}{\lambda}
  + \frac{C\lambda B^2}{n}.
\end{equation}
\end{theorem}

The bound exhibits an explicit complexity-stability trade-off in $\lambda$:
small $\lambda$ favors complexity, large $\lambda$ favors stability.
Under the additional uniformity Assumption~\ref{ass:uniformity}, taking
$\delta \to 0$ at the right rate replaces the covering number by a
data-dependent fractal dimension.

\begin{corollary}[\citealp{dupuis2023fractal, dupuis2024uniform},
data-dependent fractal dimension]\label{cor:dd-fractal}
Under Assumptions~\ref{ass:bounded}--\ref{ass:uniformity}, there exist $C>0$
and, for any $\lambda, \epsilon, \gamma > 0$, an $n_{\gamma,\epsilon}$ such
that for $n \geq n_{\gamma,\epsilon}$, with probability at least
$1 - \zeta - \gamma$:
\begin{equation}\label{eq:dd-fractal}
  \GS(\W)
  \;\leq\;
  \frac{2}{n}
  + 2B\sqrt{\frac{2(\dimB^{\vartheta_S}(\W)+\epsilon)\log(n)}{n}}
  + \frac{\log\frac{dQ_S}{dP}(\W) + \log(1/\zeta)}{\lambda}
  + \frac{C\lambda B^2}{n}.
\end{equation}
\end{corollary}

\section{Refining the geometry: topological complexity}\label{sec:topology}

The bounds of \S\ref{sec:pac-bayes} control the generalization gap through
covering numbers or fractal dimensions of $\WS$. These quantities are
informative but largely geometry-blind: they treat the trajectory as a
generic compact set and ignore its internal structure. Following
\citet{andreeva2024topological}, we now refine the complexity term using
two quantities adapted to point-cloud trajectories: the $\alpha$-weighted
lifetime sum and the positive magnitude.

\subsection{From covering numbers to topological descriptors}

Theorem~\ref{thm:dd-covering} gives a bound in terms of $|N_\delta(\WS)|$.
Specialized to a discrete-time trajectory $W_{t_0\to T} = \{w_{t_0}, \dots,
w_T\}$, and with the introduction of a mutual-information term to handle the
data-dependence of $\WS$, this gives a covering-form trajectory bound.

\begin{proposition}[\citealp{andreeva2024topological}, trajectory covering]
\label{prop:trajectory-cover}
Under Assumption~\ref{ass:bounded} and assuming $\ell$ is $(q, L, \rho)$-Lipschitz
for $q \geq 1$, for all $\delta > 0$, with probability at least $1-\zeta$
over $\mu_z^{\otimes n} \otimes \mu_u^{\otimes \infty}$:
\begin{equation}
  \sup_{t_0 \leq i \leq T} \GS(w_i)
  \;\leq\;
  2L\delta + 2B\sqrt{\frac{2\log N^\rho_\delta(W_{t_0\to T})}{n}}
  + 3B\sqrt{\frac{I_\infty(S, W_{t_0\to T}) + \log(1/\zeta)}{2n}}.
\end{equation}
\end{proposition}

Replacing the covering number by a topological descriptor of the trajectory
gives the next two main theorems.

\subsection{$\alpha$-weighted lifetime sums}

\begin{figure}[ht]
  \centering
  \includegraphics[width=0.85\textwidth]{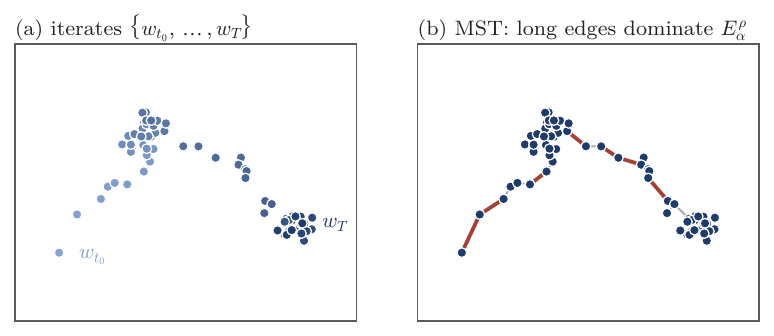}
  \caption{Trajectory geometry. (a) A simulated SGD trajectory in $\R^2$
  alternates between fast transitions and clustered phases; dot color
  encodes the iterate index $t_0,\dots,T$. (b) The minimum spanning tree
  on the same point cloud, with the top decile of edge lengths
  highlighted: a few long edges concentrate most of the mass of
  $E^\rho_\alpha = \sum_{e \in T} |e|^\alpha$, so trajectories that
  cluster contribute little.}
  \label{fig:trajectory}
\end{figure}

\begin{theorem}[\citealp{andreeva2024topological}, lifetime-sum bound]
\label{thm:lifetime}
Under Assumption~\ref{ass:bounded}, suppose $\ell$ is $(q,L,\rho)$-Lipschitz
for $q \geq 1$. For all $\alpha \in [0,1]$, with probability at least $1-\zeta$:
\begin{equation}\label{eq:lifetime}
  \sup_{t_0 \leq i \leq T} \GS(w_i)
  \;\leq\;
  2B\sqrt{\frac{2\log\bigl(1 + K_{n,\alpha}\,E^\rho_\alpha(W_{t_0\to T})\bigr)}{n}}
  + \frac{2B}{\sqrt{n}}
  + 3B\sqrt{\frac{I_\infty(S, W_{t_0\to T}) + \log(1/\zeta)}{2n}},
\end{equation}
where $K_{n,\alpha} := 2\bigl(2L\sqrt{n}/B\bigr)^\alpha$.
\end{theorem}

The bound replaces the (algorithm-agnostic) covering number by
$E^\rho_\alpha$, which directly tracks the topological complexity of the
trajectory. The intuition is that closely-clustered iterates contribute
short edges to the MST and thus a small lifetime sum.

\subsection{Positive magnitude}

A complementary, scale-sensitive quantity is the positive magnitude.
While covering numbers and lifetime sums measure ``how many points cover the
trajectory at scale $\delta$'', the magnitude $\Mag^\rho(s\W)$ measures the
\emph{effective number of points} of $\W$ as seen at the resolution
controlled by $s$. The positive part $\PMag^\rho$ is a robust variant
introduced by \citet{andreeva2024topological}.

\begin{figure}[ht]
  \centering
  \includegraphics[width=0.95\textwidth]{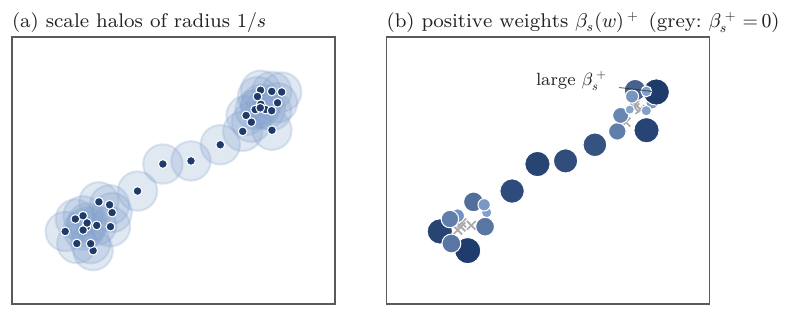}
\caption{Positive magnitude on a simulated trajectory ($s = 3$).
(a) The optimization trajectory viewed at scale $s$. The circles
illustrate the interaction distance $1/s$: iterates in clusters heavily
overlap, whereas transition points remain isolated.
(b) The size and color intensity of each iterate encode its positive
weight $\beta_s(w_i)^+$. Clustered iterates ``shadow'' each other and
receive near-zero weights, contributing little to $\PMag^\rho$.
Conversely, isolated iterates along the transition path act as
independent effective points and receive large weights.}
  \label{fig:magnitude}
\end{figure}

\begin{theorem}[\citealp{andreeva2024topological}, positive-magnitude bound]
\label{thm:pmag}
Suppose the loss $\ell$ is $(q,L,\rho)$-Lipschitz for $q \geq 1$ and that
$(\W, \lambda\rho)$ admits a positive magnitude for every $\lambda > 0$.
For any $s > 0$, with probability at least $1-\zeta$:
\begin{equation}\label{eq:pmag}
  \sup_{t_0 \leq i \leq T} \GS(w_i)
  \;\leq\;
  \frac{2}{s}\log \PMag^\rho(Ls\,W_{t_0\to T})
  + \frac{sB^2}{n}
  + 3B\sqrt{\frac{I_\infty(S, W_{t_0\to T}) + \log(1/\zeta)}{2n}}.
\end{equation}
\end{theorem}

Theorems~\ref{thm:lifetime} and \ref{thm:pmag} share the same residual
mutual-information term $I_\infty(S, W_{t_0\to T})$. This term is what
\S\ref{sec:stability} eliminates.

\section{Removing intractable terms: algorithmic stability}\label{sec:stability}

The bounds of \S\ref{sec:pac-bayes}--\S\ref{sec:topology} all carry a
data-dependence term: $\KL(Q_S\|P)$ for the PAC-Bayesian bounds and
$I_\infty(S, \WS)$ for the topological ones. Both are well-defined
information-theoretic (IT) quantities, but in practice they are essentially
intractable for discrete-time SGD or Adam on deep networks. For continuous
Langevin dynamics one can control them analytically; for the optimizers
actually used in the field, no usable bound is known. The bounds of
\S\ref{sec:topology} are therefore theoretically clean but practically
opaque.

Algorithmic stability \citep{bousquet2002stability, bassily2020stability}
offers a way out: it characterizes the algorithm itself, rather than its
information-theoretic interaction with $S$. The key insight of
\citet{tuci2026stability} is to lift the standard notion of uniform argument
stability (Assumption~\ref{ass:uas}) to a property of the entire random set
$\WSU$ (Assumption~\ref{ass:traj}). Under
Assumptions~\ref{ass:bounded}, \ref{ass:traj}, and \ref{ass:lip-rs}, the IT
term can then be replaced by a stability constant $\beta_n$.

\subsection{IT-free fractal-dimension bound}

\begin{theorem}[\citealp{tuci2026stability}, IT-free fractal bound]
\label{thm:it-free-fractal}
Under Assumptions~\ref{ass:bounded}, \ref{ass:traj}, and \ref{ass:lip-rs},
suppose $\WSU$ is almost surely of finite diameter, and (without loss of
generality) that $\beta_n^{-2/3}$ is an integer divisor of $n$. There exists
$\delta_n > 0$ such that for all $\delta < \delta_n$,
\begin{equation}
  \E\!\left[\sup_{w \in \WSU}\bigl(\Risk(w) - \Empirical_S(w)\bigr)\right]
  \;\leq\;
  2\,\E\!\left[
    \tfrac{B}{n}
    + \delta L_{S,U}
    + \beta_n^{1/3}\!\left(1 + B\sqrt{4\,\dimB^{\vartheta_S}(\WSU)\log\tfrac{1}{\delta}}\right)
  \right].
\end{equation}
\end{theorem}

\subsection{IT-free topological bounds}

\begin{theorem}[\citealp{tuci2026stability}, IT-free topological bounds]
\label{thm:it-free-topo}
Under Assumptions~\ref{ass:bounded}, \ref{ass:traj}, and \ref{ass:lip-rs},
assume that $\WSU$ is almost surely finite and that $\beta_n^{-2/3}$ is an
integer divisor of $n$. Then for any $\alpha \in (0,1]$,
\begin{equation}\label{eq:it-free-lifetime}
  \E\!\left[\sup_{w \in \WSU}\bigl(\Risk(w) - \Empirical_S(w)\bigr)\right]
  \;\leq\;
  \beta_n^{1/3}\!\left(2 + 2B + 2B\,\E\!\sqrt{2\log\bigl(1 + K_{n,\alpha} E_\alpha(\WSU)\bigr)}\right),
\end{equation}
where $K_{n,\alpha} := 2(2L_{S,U}\sqrt{n}/B)^\alpha$. Moreover, for any
$\lambda > 0$, with $s(\lambda) := \lambda L_{S,U}\beta_n^{-1/3}/B$,
\begin{equation}\label{eq:it-free-pmag}
  \E\!\left[\sup_{w \in \WSU}\bigl(\Risk(w) - \Empirical_S(w)\bigr)\right]
  \;\leq\;
  \beta_n^{1/3}\!\left(2 + \lambda B + \tfrac{2B}{\lambda}\,
    \E\!\bigl[\log \PMag(s(\lambda)\cdot \WSU)\bigr]\right).
\end{equation}
\end{theorem}

The two bounds \eqref{eq:it-free-lifetime}--\eqref{eq:it-free-pmag} are
IT-free counterparts of Theorems~\ref{thm:lifetime} and \ref{thm:pmag}: the
$I_\infty$ term has been absorbed into the stability constant $\beta_n$
through Assumption~\ref{ass:traj}, while preserving the topological complexity
in $E_\alpha$ and $\PMag$.

\section{Synthesis and conclusion}\label{sec:synthesis}

Across PAC-Bayesian theory, fractal geometry, persistent homology, and
algorithmic stability, a single template inequality~\eqref{eq:template}
governs all the bounds we have surveyed. Table~\ref{tab:bounds} makes the
contributions of each line of work directly comparable along five axes:
the setting (fixed $\W$, random $\WS$, algorithm-dependent $\WSU$), the
complexity term, the IT term, the type of guarantee (uniform, expectation,
disintegrated), and the assumptions invoked. Every term that vanishes from
one row to the next signals a real conceptual or practical gain.

\begin{table}[t]
\label{tab:bounds}
\centering
\footnotesize
\setlength{\tabcolsep}{4pt}
\renewcommand{\arraystretch}{1.15}
\resizebox{\textwidth}{!}{%
\begin{tabular}{@{}lllllll@{}}
\toprule
Theorem & Setting & Complexity & IT term & Type & Assumptions \\
\midrule
Th.\,\ref{thm:minkowski} \citep{simsekli2020hausdorff}
  & fixed $\W$, finite diam.\ & $\dim_M(\W)$ & -- & uniform & A1 \\
Th.\,\ref{thm:rademacher} \citep{mohri2018foundations}
  & fixed $\W$ & $\Rad_S(\W)$ & -- & uniform & A1 \\
\midrule
Th.\,\ref{thm:classical-pb}\,(i,ii) \citep{germain2009pac}
  & fixed $\W$, randomized $Q_S$ & abstract $\phi$ & $\KL$ / $\log\frac{dQ_S}{dP}$
  & exp.\ / disint. & -- \\
Th.\,\ref{thm:master} \citep{dupuis2024uniform}
  & random $\W$ & abstract $\Phi$ & $\KL$ / $\log\frac{dQ_S}{dP}$
  & exp.\ / disint. & A2 \\
Th.\,\ref{thm:dd-rademacher} \citep{dupuis2024uniform}
  & random $\W$ & $\Rad_S(\W)$ & $\KL$ / $\log\frac{dQ_S}{dP}$
  & exp.\ / disint. & A1, A2 \\
Th.\,\ref{thm:dd-covering} \citep{dupuis2024uniform}
  & random $\W$ & $\log|N_\delta(\W)|$ & $\log\frac{dQ_S}{dP}$
  & disint. & A1--A3 \\
Cor.\,\ref{cor:dd-fractal} \citep{dupuis2023fractal, dupuis2024uniform}
  & random $\W$ & $\dimB^{\vartheta_S}(\W)$ & $\log\frac{dQ_S}{dP}$
  & disint. & A1--A4 \\
\midrule
Prop.\,\ref{prop:trajectory-cover} \citep{andreeva2024topological}
  & trajectory & $\log N^\rho_\delta(W_{t_0\to T})$ & $I_\infty$
  & high-prob.\ & A1, Lipschitz \\
Th.\,\ref{thm:lifetime} \citep{andreeva2024topological}
  & trajectory & $E^\rho_\alpha(W_{t_0\to T})$ & $I_\infty$
  & high-prob.\ & A1, Lipschitz \\
Th.\,\ref{thm:pmag} \citep{andreeva2024topological}
  & trajectory & $\PMag^\rho(Ls\cdot W_{t_0\to T})$ & $I_\infty$
  & high-prob.\ & A1, Lipschitz \\
\midrule
Th.\,\ref{thm:it-free-fractal} \citep{tuci2026stability}
  & $\WSU$ & $\dimB^{\vartheta_S}(\WSU)$ & \textbf{none} & expectation & A1, A6, A7 \\
Th.\,\ref{thm:it-free-topo} \citep{tuci2026stability}
  & $\WSU$ & $E_\alpha(\WSU)$, $\PMag(\cdot\,\WSU)$ & \textbf{none} & expectation & A1, A6, A7 \\
\bottomrule
\end{tabular}}
\caption{Comparison of the bounds discussed in this survey. ``--'' denotes
absence of the term. ``Disint.'' = disintegrated.}
\end{table}

\paragraph{When to use what.}
The reader's choice depends on which terms in
template~\eqref{eq:template} they can control:
\begin{itemize}[leftmargin=2em]
\item \emph{If a tight prior is available}, the PAC-Bayesian
  bounds of \S\ref{sec:pac-bayes} (Theorem~\ref{thm:dd-covering}
  and Corollary~\ref{cor:dd-fractal}) are the sharpest, with
  $\KL(Q_S\|P)$ as the data-dependence term.
\item \emph{If only the trajectory point-cloud is available}, the topological
  bounds of \S\ref{sec:topology} (Theorems~\ref{thm:lifetime}
  and~\ref{thm:pmag}) are the most natural, but require controlling
  $I_\infty$.
\item \emph{If $I_\infty$ cannot be controlled} (e.g., for Adam, ~\cite{kingma2014adam}), the
  stability-based bounds of \S\ref{sec:stability}
  (Theorems~\ref{thm:it-free-fractal} and~\ref{thm:it-free-topo}) are the
  only currently known route, at the price of an algorithm-side stability
  assumption.
\end{itemize}

\acks{This work was conducted as part of the EA research program at \'Ecole
polytechnique under the supervision of Umut \c{S}im\c{s}ekli and Benjamin
Dupuis. We acknowledge the support of INRIA and the resources provided by
\'Ecole polytechnique.}

\bibliography{refs}

\end{document}